\documentclass[letterpaper,10pt,conference]{ieeeconf}

\IEEEoverridecommandlockouts
\usepackage{graphicx}
\usepackage{amsmath,amssymb,bbm}
\usepackage{xcolor}
\usepackage{optidef}
\usepackage[normalem]{ulem}
\usepackage{mathtools}
\usepackage{cuted}
\usepackage{algorithm}
\usepackage{algpseudocode}
\usepackage{breqn}
\usepackage{subfig}
\usepackage{hyperref}
\usepackage{booktabs}
\usepackage{cite}
\usepackage[T1]{fontenc}
\usepackage[utf8]{inputenc}

\newtheorem{remark}{Remark}
\newtheorem{theorem}{Theorem}
\newtheorem{assumption}{Assumption}
\newtheorem{lemma}{Lemma}
\newtheorem{proposition}{Proposition}

\newtheorem{definition}{\hspace{0pt}\bf Definition}

\newcommand\redsout{\bgroup\markoverwith{\textcolor{red}{\rule[0.5ex]{2pt}{0.4pt}}}\ULon}
\DeclareMathOperator*{\argmax}{argmax}
\definecolor{mygreen}{rgb}{0.10,0.50,0.10}

\title{\LARGE \bf Safety Guarantees in Zero-Shot Reinforcement Learning for Cascade
Dynamical Systems}

\title{\LARGE \bf Safety Guarantees in Zero-Shot Reinforcement Learning for Cascade
Dynamical Systems}

\author{Shima Rabiei$^\dagger$, Sandipan Mishra$^\star$ and Santiago Paternain$^\dagger$
\thanks{$^\dagger$The authors are with the Department of Electrical, Computer, and Systems Engineering, Rensselaer Polytechnic Institute. Email: \{rabiei, paters\}@rpi.edu
\newline
\indent $^\star$The author is with the Department of Mechanical, Aerospace and Nuclear Engineering, Rensselaer Polytechnic Institute. Email: mishrs2@rpi.edu
}}

\begin{document}

\maketitle
\thispagestyle{empty}
\pagestyle{empty}
\begin{center}
\begin{minipage}{0.97\linewidth}
\footnotesize\textit{This work has been submitted to the IEEE for possible publication.
Copyright may be transferred without notice, after which this version may no longer be accessible.}
\end{minipage}
\end{center}

\begin{abstract}
This paper considers the problem of zero-shot safety guarantees for cascade dynamical systems. These are systems where a subset of the states (the inner states) affects the dynamics of the remaining states (the outer states) but not vice-versa. We define safety as remaining on a set deemed safe for all times with high probability. 
We propose to train a safe RL policy on a reduced-order model, which ignores the dynamics of the inner states, but it treats it as an action that influences the outer state. Thus,  reducing the complexity of the training. When deployed in the full system the trained policy is combined with a low-level controller whose task is to track the reference provided by the RL policy. Our main theoretical contribution is a bound on the safe probability in the full-order system. In particular, we establish the interplay between the probability of remaining safe after the zero-shot deployment and the quality of the tracking of the inner states. We validate our theoretical findings on a quadrotor navigation task, demonstrating that the preservation of the safety guarantees is tied to the bandwidth and tracking capabilities of the low-level controller.
\end{abstract}

\section{Introduction}
 
Reinforcement learning (RL) has demonstrated remarkable empirical success in simulated environments, from mastering classic Atari \cite{mnih2015dqn} and board games \cite{silver2016alphago,silver2017alphagozero} to solving high-dimensional control benchmarks \cite{schulman2017ppo}; however, deploying these policies on safety-critical hardware remains difficult. Collecting data on hardware takes inordinate time and can result in unsafe behaviors. Thus, it is not uncommon to train policies in simulation and then apply them to the real system. However, the computational burden of simulating high-dimensional systems makes policy optimization sample-inefficient. Because of these reasons, it is standard practice to simplify the training environment. This is typically achieved by identifying reduced-order models \cite{chen2024rom,zhao2020simtoreal} or by exploiting cascade structures to decouple low-level dynamics \cite{11107987}. While reduced-order training makes learning tractable, it introduces a transfer gap. A policy that is safe and performant on a simplified model may easily violate constraints when deployed on the full-order system. This discrepancy is a severe bottleneck for safety-sensitive applications. 

Existing literature tackles safe RL from several angles: constrained and probabilistic formulations enforce safety during training on a nominal model \cite{achiam2017cpo,9718160,chen2024probabilistic}, while barrier functions and predictive filters act as runtime safeguards to overwrite unsafe actions \cite{cheng2019barrier,wabersich2021predictive}. However, constructing valid safety barrier functions or computing safety projections for high-dimensional full-order systems is challenging. This difficulty stems from the complex nonlinearities of the full-order dynamics and the analytical burden of ensuring forward invariance while respecting strict actuator limits \cite{molnar2023safetycritical}. Furthermore, while transfer-oriented approaches analyze safety under sim-to-real mismatch and robust dynamic shifts \cite{hsu2023simlabreal,zhang2024safetytransfer,as2025spidr,zubia2025robust}, these works typically rely on deterministic bounding tubes under bounded disturbances or they focus on satisfying cumulative expected cost constraints. While the performance due to the mismatch between dynamics degrades gracefully in the case of cumulative constraints, this is not necessarily the case in strict safety settings. Indeed, a policy that remains in safe set for all times on the reduced order model can violate the constraint in the full-order model after transferring. 
Thus the question of whether a probabilistic safety certificate obtained on a reduced-order model can be quantitatively transferred to a higher-order deployment model remains unanswered.

Cascade dynamical systems offer a natural architecture to answer this question. In a cascade structure, outer-loop task variables depend on an inner subsystem, which can typically be stabilized by a classical tracking controller. This motivates a hierarchical learning approach: an RL agent generates outer-loop actions alongside inner-state references, while a low-level controller tracks these references on the full-order plant. Our previous work \cite{11107987} exploited this structure to bound transfer-induced \emph{performance} degradation. Safety, however, as described above, poses additional challenges. 

The contribution of this paper is a framework for transferring safety guarantees from a reduced-order model to a full-order cascade system. Specifically, we (i) frame the safe reduced-order training as a constrained RL problem; and (ii) we establish a safety guarantee on the zero-shot deployment of the policy, which  quantifies how safety degrades as a function of inner-loop tracking quality, initial tracking error, and the variation in the reference sequence. 

The remainder of the paper is organized as follows. Section~\ref{sec_problem_formulation} introduces the full-order cascade system and formulates the safe RL problem. Section~\ref{sec_reduced_training} presents the reduced-order training setup, while Section~\ref{sec_transfer_safety} the main theoretical result on the zero-shot safety guarantees. Section~\ref{sec:analysis} provides the analysis and proofs. Other than concluding remarks (Section~\ref{sec_conclusion}), we illustrate our framework in  Section~\ref{sec_numerical_examples}, through a quadrotor navigation problem.

\section{Problem Formulation}\label{sec_problem_formulation}
In this paper, we study the zero-shot safety guarantees of a policy learned on a reduced-order model when deployed on a full-order cascade dynamical system.   

Let us denote the time index by $t\in\mathbb{N}$, $(S_t,X_t)\in\mathcal{S}\times\mathcal{X}$ and $(A_t,U_t)\in\mathcal{A}\times\mathcal{U}$ denote, respectively, the state and input of the full-order system at time~$t$.  Here, $S_t\in\mathcal{S}\subseteq\mathbb{R}^n$ and $A_t\in\mathcal{A}\subseteq\mathbb{R}^p$ are the outer state and outer input, while $X_t\in\mathcal{X}\subseteq\mathbb{R}^m$ and $U_t\in\mathcal{U}\subseteq\mathbb{R}^q$ are the inner state and inner input. Let $\Delta(\cdot)$ denote the set of probability measures on its argument, and let $\mathbb{P}_{F}:(\mathcal{S}\times\mathcal{X})\times(\mathcal{A}\times\mathcal{U})\to\Delta(\mathcal{S}\times\mathcal{X})$ denote the transition kernel of the full-order system.
%%%%%%%%%% an example of cascade connection %%%%%%%%%%
A motivating example of this structure is the quadrotor system considered in Section~\ref{sec_numerical_examples}. In that setting, the outer state corresponds to the translational variables, while the inner state consists of the attitude variables and their angular rates. This separation stems from the fact that the attitude subsystem influences the translational motion, whereas the translational variables do not directly influence the attitude dynamics. The following assumption formalizes this cascade connection at the level of the transition kernel. 

\begin{assumption}\label{assump:cascade_structure}
For all $(s,x,a,u)\in\mathcal{S}\times\mathcal{X}\times\mathcal{A}\times\mathcal{U}$,
\begin{align}
    \mathbb{P}_{F}(X_{t+1}\mid S_t=s,&X_t=x,A_t=a,U_t=u)
    \\&=
    \mathbb{P}_{F}(X_{t+1}\mid X_t=x,U_t=u).\nonumber
\end{align}
\end{assumption}

Assumption~\ref{assump:cascade_structure} implies that, once the current inner state and inner input are fixed, the inner-loop transition is independent of the outer variables.

Let us define a safe set $\mathcal{S}_{\texttt{safe}}\subset\mathcal{S}$. Our objective is to train policies that are deemed safe in the sense that they remain in the safe set for all times with high probability. %i.e.,

Note that the safe set does not include the internal states $\mathcal{X}$. We discuss this choice in Remark \ref{remark_safety}.In the quadrotor example in Section \ref{sec_numerical_examples}, the safe set $\mathcal{S}_{\mathrm{safe}}$ captures admissible translational behavior, for instance by requiring the vehicle position to remain within a prescribed spatial region, e.g., obstacle avoidance. Since guaranteeing positive invariance over an infinite horizon is generally infeasible for stochastic systems unless the noise is bounded (as a counterexample consider a linear system with Gaussian noise). Therefore, we relax the objective to guaranteeing safety until a finite operational horizon $T > 0$. To be formal, define the control policy $\pi$ as a conditional probability measure $\pi:\left(\mathcal{S}\times \mathcal{X}\right)\to\Delta(\mathcal{A}\times\mathcal{U})$ and let $1-\delta$, with $\delta<1$, be the desired safety probability. Then, a safe policy is one satisfying 
\begin{equation}\label{eq:chance_constraint}
\mathbb{P}\left( \bigcap_{t=0}^{T} {S_t \in \mathcal{S}_{\mathrm{safe}}} \mid \pi \right) \geq 1-\delta.
\end{equation}

In addition to safety, we are interested in optimizing the performance of the policy $\pi$. Let $r:\mathcal{S}\times\mathcal{A}\times\mathcal{X}\times\mathcal{U}\to\mathbb{R}$ be the instantaneous reward, which we assume bounded. Further define the expected cumulative return of a policy as
\begin{align}\label{eq:reward_value_F}
    J_r^F(\pi)
    := \mathbb{E}_{F}^{\pi}\!\left[\sum_{t=0}^{T} r(S_t,A_t,X_t,U_t)\right].
\end{align}

With these definitions the safe reinforcement learning problem is given by 

\begin{align}
    \max_{\pi}\quad & J_r^F(\pi) \nonumber \\
    \text{s.t.}\quad & \mathbb{P}\left( \bigcap_{t=0}^{T} \{S_t \in \mathcal{S}_{\mathrm{safe}}\} \mid \pi \right) \geq 1-\delta.
\end{align}
Solving the above problem in the full-order system using data-driven approaches such as reinforcement learning carries several challenges, of which the large dimensionality of the state and action space and the long running time of simulations of a complex dynamical system are paramount. For these reasons we consider training in a simplified reduced order model. We detail the training procedure and the assumptions on the reduced-order model in the next section.

\section{Reduced-order model training}\label{sec_reduced_training}

The cascade structure motivates us to introduce a reduced-order model in which the inner-state variable is regarded as a reference input. Specifically, the reduced-order model has state $S_t\in\mathcal{S}$ and input $(A_t,X_t^\star)\in\mathcal{A}\times\mathcal{X}$, where $X_t^\star$ denotes an inner-state reference. Let $\mathbb{P}_{R}:\mathcal{S}\times(\mathcal{A}\times\mathcal{X})\to\Delta(\mathcal{S})$ denote its transition kernel. Thus, the random variable $X_t$ appears as an inner state in the full-order model and as an input in the reduced-order model.

%%%%%%%%%% Pr Pf agrees %%%%%%%%%%
We assume that for any realization $x$ of the random variable $X_t$, the reduced-order and full-order models induce the same outer-state transition. This is formalized by the following assumption.
\begin{assumption}\label{assump:outer_matching}
For all $(s,x,a,u)\in\mathcal{S}\times\mathcal{X}\times\mathcal{A}\times\mathcal{U}$,
\begin{align}
    \mathbb{P}_{F}(S_{t+1}\mid & S_t=s,X_t=x,A_t=a,U_t=u)
    \\ & \nonumber =
    \mathbb{P}_{R}(S_{t+1}\mid S_t=s,A_t=a,X_t^\star=x).
\end{align}
\end{assumption}
% \blue{Discuss connection with the quadrotor.}
%
In the quadrotor example, the outer dynamics correspond to the linear positions and velocities, which depend on its attitude (see Section \ref{sec_numerical_examples}). 
Thus, if the attitude of the full-order system (the inner state) is equal to the commanded reference of the reduced-order model, the resulting force vectors acting on the vehicle's center of mass are identical. Therefore, resulting in the same translational dynamics.

Given the differences between the full-order and reduced order model, we redefine the policy to be a conditional probability  $\pi:\mathcal{S}\to\Delta(\mathcal{A}\times\mathcal{X})$, which selects the outer input and inner-state reference $(A_t,X_t^\star)$ based on the outer state $S_t$. In the quadrotor setting, the policy observes only the translational position and velocity. Based on this outer state, the agent outputs a commanded thrust and a desired attitude reference. Consequently, the agent neither observes nor directly actuates the instantaneous physical attitude and angular rates, leaving the rotational dynamics to be handled entirely by the inner loop.
Thus, we evaluate the reward using the commanded reference $X_t^\star$ instead of the inner state $X_t$. We define, with a small abuse of notation,
an instantaneous reward function $r:\mathcal{S}\times\mathcal{A}\times\mathcal{X}\to\mathbb{R}$ independent of $U$, which we also assume to be bounded. 
We must point out that our choice of reduced-order modeling prevents us from considering rewards dependent on the low-level action $U_t$.  

With these definitions we define the performance of policy $\pi$ in the reduced order model as 
\begin{equation}\label{eq:reward_value}
    J_r^R(\pi)
    := \mathbb{E}_{R}^{\pi}\!\left[\sum_{t=0}^{T}  r(S_t,A_t,X_t^\star)\right].
\end{equation}

To incorporate the safety requirements \eqref{eq:chance_constraint}, we follow the approach in~\cite{9718160}, where a safety cost is defined as 
\begin{align}
    c(s) := \mathbf{1}\{s\notin\mathcal{S}_{\mathrm{safe}}\},
\end{align}
which assigns an instantaneous cost whenever the state lies outside the safe set. We then formalize safety through the expected undiscounted cumulative cost under policy $\pi$,
% \begin{align}\label{eq:safety_value}
%     J_c^R(\pi)
%     := \mathbb{E}_{R}^{\pi}\!\left[\sum_{t=0}^{T}  c(S_t)\right].\nonumber
% \end{align}
\begin{equation}\label{eq:safety_value}
J_c^R(\pi)
:= \mathbb{E}_{R}^{\pi}\!\left[\sum_{t=0}^{T} c(S_t)\right].
\end{equation}
Note that since  $c(s)\in\{0,1\}\, \forall s\in\mathcal{S}$, $J_c^R(\pi)$ is bounded.

With these definitions, we formulate the training on the reduced-order model as solving the following  constrained Markov Decision Process
\begin{align}\label{eq:reduced_cmdp}
    \pi_R^\star \in \argmax_{\pi}\quad & J_r^R(\pi) \nonumber \\
    \text{s.t.}\quad & J_c^R(\pi)\le \delta.
\end{align}

The formulation guarantees that $\pi_R^\star$ is  safe in the reduced-order model in the sense of \eqref{eq:chance_constraint} (see Section \ref{sec:analysis}).

A common approach to solving Constrained Reinforcement Learning problems, is through their Lagrangian relaxation, (see, e.g.,  \cite{altman1999constrained,bhatnagar2005actor,paternain2019constrainedreinforcementlearningzero,tessler2018reward})
Let $\lambda \ge 0$ denote the Lagrange multiplier for the safety constraint, and define the Lagrangian by
\begin{align}\label{lagrangian}
    \mathcal{L}(\pi,\lambda)
    :=
    J_r^R(\pi)-\lambda\bigl(J_c^R(\pi)-\delta\bigr).\nonumber
\end{align}
This transforms \eqref{eq:reduced_cmdp} into the minimax problem
\begin{align}
    \min_{\lambda\ge 0}\max_{\pi}\; \mathcal{L}(\pi,\lambda).
\end{align}
Note that for a fixed multiplier $\lambda$, maximizing the Lagrangian with respect to the policy $\pi$ is equivalent to solving a standard unconstrained RL problem whose reward is given by.
\begin{align}
    \tilde r(s,a,x^\star,\lambda)
    :=
    r(s,a,x^\star)-\lambda c(s).
\end{align}
Thus, The resulting problem can then be addressed by gradient ascent methods in the policy parameters $\theta$ and projected gradient descent in the multiplier $\lambda$.
See Section~\ref{sec_numerical_examples} for the implementation details.

While $\pi_R^\star$ is safe in the reduced-order model, it is not guaranteed to be safe in the full-order model, since the former ignores the dynamics of $X$. % 
Moreover, the policy is not directly implementable on the full-order system since it outputs an inner-state reference $X_t^\star$ rather than the inner-loop input $U_t$. We discuss the details of the transfer and the safety guarantees in the following section. Before doing so, we present a remark regarding the definition of the safe set.

\begin{remark}\label{remark_safety}
 Since we train the policy in the reduced order model, the reference signal $X_t^\star$ is interpreted as an input. Therefore, guaranteeing that it remains in pre-specified sets can be attained by add-hoc methods. For example, if the reference for the pitch angle of the quadrotor is larger than a given desired value, it is possible to project it onto a desired set. As such, these modifications can be induced directly in the policy parameterization. We would be remiss  not to point out that ``safe'' reference does not imply that the inner state $X_t$ is within the desired set. However, the low-level controller we define in the next section could address this guarantee.     
\end{remark}

\section{Transfer Safety Guarantees}\label{sec_transfer_safety}
We start this section by defining a controller whose goal is for the inner state $X_t$ to track the reference signal $X_t^\star$
\begin{align}\label{eq:conrol_law}
    U_t = K(X_t^\star,X_t).
\end{align}

%%%%%%%%%%%% introduce the closed loop kernel / MDP
That is, the controller maps the reference and state pair $(X_t^\star,X_t)$ to the inner-loop input $U_t$. 

The resulting closed-loop full-order system induces a Markov process with transition kernel 
\begin{align}
    &\mathbb{P}_{K}\bigl((S_{t+1},X_{t+1}) \mid (S_t,X_t),(A_t,X_t^\star)\bigr)
    \\ & \nonumber:=
    \mathbb{P}_{F}\bigl((S_{t+1},X_{t+1}) \mid (S_t,X_t),(A_t,K(X_t^\star,X_t))\bigr).
\end{align}
The subscript $K$ indicates that the closed-loop transition kernel is induced by the inner-loop controller $K$.
Let $\mathbb{P}_K$ be the marginal outer-state transition probability induced by the closed-loop kernel, i.e.,
\begin{align}
    \mathbb{P}&_{K,t}^{S}\bigl(S_{t+1} \mid S_t,A_t,X_t^\star\bigr) :=  \\
    &\sum_{X_t,X_{t+1}\in\mathcal{X}} \mathbb{P}_{K}\bigl((S_{t+1},X_{t+1}) \mid (S_t,X_t),(A_t,X_t^\star)\bigr) \mathbb{P}_{K}(X_t).\nonumber
\end{align}

For notational simplicity, we omit the superscript $S$ and the dependence on $t$, writing $\mathbb{P}_{K}(\cdot\mid s,a,x^\star)$ whenever no confusion arises.

We assume that the closed-loop system is stable in the following sense.

 \begin{assumption}\label{assump:tracking_cdc} 
Let us define the inner loop tracking error 
\begin{equation}\label{eqn_tracking_error}
e_t := \mathbb{E}\!\left[\|X_t - X_t^\star\|_{P}\right] .\nonumber
\end{equation}
There exist a matrix $P \in \mathbb{S}^{m\times m}_{++}$ and constants $\alpha \in (0,1)$ and $\beta > 0$ such that, under the controller defined in \eqref{eq:conrol_law}, the tracking error $e_t$ satisfies
\begin{align} 
e_t 
\le 
\alpha\,e_{t-1} 
+\beta\,\mathbb{E}\!\left[\|X_t^\star-X_{t-1}^\star\|_{P}\right] 
\label{eq:tracking_assump_cdc} 
\end{align} 
for all $t \ge 1$.
\end{assumption} 
Assumption~\ref{assump:tracking_cdc} establishes that the closed-loop tracking error dynamics %operate as a contraction in expectation when the reference is constant. More broadly, it asserts that the inner loop is
are input-to-state stable (ISS) in expectation with respect to the consecutive variations of the reference signal $X_t^\star$. By iterating \eqref{eq:tracking_assump_cdc}, we obtain the trajectory bound
\begin{align}\label{eq:iterated_tracking} 
e_t
\le 
\alpha^{t} e_0 
+\beta\sum_{\ell=1}^{t}\alpha^{t-\ell} 
\mathbb{E}\!\left[\|X_\ell^\star-X_{\ell-1}^\star\|_{P}\right]. 
\end{align}

In the context of ISS formulations (see e.g.,~\cite{Khalil:1173048}), the first term on the right-hand side of \eqref{eq:iterated_tracking} represents a class $\mathcal{KL}$ function, characterizing the transient exponential decay of the initial error. The summation term results in a class $\mathcal{K}$ function on the supremum of the reference variations. Notice that Assumption~\ref{assump:tracking_cdc} alone does not guarantee a uniformly bounded tracking error unless the sequence of reference variations is also bounded. To address this, we impose the following assumption on the reference sequence.

\begin{assumption}\label{assump:reference_variation_cdc}
There exists a nonnegative sequence \(\{d_t\}_{t=1}^{T}\) and a constant \(D>0\) such that
\begin{align}
\mathbb{E}\!\left[\|X_t^\star-X_{t-1}^\star\|_{P}\right] \le d_t\, 
\end{align}
for all \(t\ge 1\), and
\begin{align}\label{eqn_sum_dt}
\sum_{t=1}^{T}  d_t \le D.
\end{align}
\end{assumption}

Assumption~\ref{assump:reference_variation_cdc} makes no assumptions about the learning algorithm itself. It requires only that the reference trajectories generated by the optimal policy have a bounded expected rate of change. Given that the reference space $\mathcal{X}$ is naturally bounded in physical applications, imposing limits on $X_t^\star$ directly guarantees the existence of such a sequence. We also point out that the summation in \eqref{eqn_sum_dt} starts at \(t=1\) since \(X_t^\star-X_{t-1}^\star\) is defined only for \(t\ge 1\).

With these definitions, we move towards bounding the mismatch between transition kernels of the reduced and full order systems. To do so, we resort to the total variation distance (see~\cite[Chap.~3]{DBLP:books/cu/10/D2010}).
\begin{definition}\label{def:delta_cdc}
For any states $s, s'\in\mathcal{S}$,  input $a\in\mathcal{A}$, and inner-state reference $x^\star\in\mathcal{X}$, denote the difference between the closed-loop and reduced-order transition kernels by
\begin{align}\label{eq:delta_P_def}
\Delta\mathbb{P}(s'\mid s,a,x^\star) 
:= 
\mathbb{P}_{K}(s'\mid s,a,x^\star) - \mathbb{P}_{R}(s'\mid s,a,x^\star).
\end{align}
For each $t\ge 0$, we then define the maximum one-step total variation mismatch as
\begin{align}\label{eq:delta_t_def}
\Delta_{t+1}
:= 
\sup_{s\in\mathcal{S},\,a\in\mathcal{A},\,x^\star\in\mathcal{X}} \left\| \Delta\mathbb{P}(\cdot\mid s,a,x^\star) \right\|_{\mathrm{TV}},
\end{align}
where $\|\cdot\|_{\mathrm{TV}}$ denotes the total variation norm.
\end{definition}
Recall that the reduced order model realizes the same dynamics as the outer dynamics (see Assumption \ref{assump:outer_matching}). Thus, $\Delta_{t+1}$ depends on the quality of the tracking error \eqref{eqn_tracking_error}. To formalize this result, we require the following assumptions. 

\begin{assumption}\label{assump:initial_state}
Let $\mu_R$ and $\mu_K$ denote the initial outer-state distributions of the reduced-order and full-order closed-loop systems, respectively. We assume $\mu_R(s_0) \equiv \mu_K(s_0)\, \forall s_0 \in \mathcal{S}$.
\end{assumption}

\begin{assumption}\label{assump:lipschitz_PR}
There exists a constant $L>0$ such that, for all $s\in\mathcal{S}$, $a\in\mathcal{A}$, and $x,x'\in\mathcal{X}$,
\begin{align}
    \left\|
    \mathbb{P}_{R}(\cdot\mid s,a,x')
    -
    \mathbb{P}_{R}(\cdot\mid s,a,x)
    \right\|_{\mathrm{TV}}
    \le
    L\|x'-x\|_P.
\end{align}
\end{assumption}
Assumption \ref{assump:initial_state} is of technical nature and it is made to simplify the analysis. Analogous results could be derived albeit with the added complexity of accounting for the distance in the initial state's distributions. 

Assumption \ref{assump:lipschitz_PR} is mild and generally met in practice. For instance, the total variation distance between transition probabilities of dynamical systems subject to Gaussian noise is bounded by a function proportional to the Euclidean distance between their means~\cite{devroye2018total}. Thus, this assumption naturally holds as long as the expected dynamics (the mean transitions) are Lipschitz continuous. This is standard for most mechanical and physical processes, including the quadrotor system evaluated in Section~\ref{sec_numerical_examples}.

We are now in conditions of bounding the transitions mismatch in Definition \ref{def:delta_cdc} in terms of inner-loop tracking performance. This is the subject of the next proposition.
\begin{proposition}\label{prop:delta_bound_cdc}
Under Assumptions~\ref{assump:cascade_structure}--\ref{assump:lipschitz_PR}, %Then $\Delta_1 \le L  \, e_0$, and for all $t \ge 1$,
it holds that
\begin{align}
\Delta_{t+1}
\le
L  \left(
\alpha^t e_0 + \beta \sum_{\ell=1}^{t}\alpha^{t-\ell} d_\ell
\right),
\label{eq:delta_t_cdc}
\end{align}
for all $t\ge 0$, where the summation term is understood to be zero when $t=0$. 
\end{proposition}
\begin{proof}
The derivation follows steps analogous to those in \cite[Proposition~1]{11107987}. Specifically, by applying the law of total probability, the outer-state matching condition (Assumption~\ref{assump:outer_matching}), and the Lipschitz continuity of the reduced-order kernel (Assumption~\ref{assump:lipschitz_PR}), we establish that the maximum one-step total variation mismatch is bounded by the expected inner-loop tracking error:
\begin{align}
    \Delta_{t+1} \le L  \,\mathbb{E}\!\left[\|X_t-X_t^\star\|_{P}\right].
\end{align}
The result then follows immediately by substituting the iterated tracking bound \eqref{eq:iterated_tracking} alongside the sequence bound $\mathbb{E}\!\left[\|X_\ell^\star-X_{\ell-1}^\star\|_{P}\right] \le d_\ell$ from Assumption~\ref{assump:reference_variation_cdc} directly into this inequality.\hfill \QED
\end{proof}
With the one-step transition mismatch bounded, we establish the main result of this paper: a finite-horizon safety guarantee for the transferred policy. The following theorem bounds the probability that the deployed closed-loop system violates the safety constraints.
\begin{theorem}[Finite-Horizon Transfer Safety Guarantee]\label{thm:probabilistic_safety_new}
Let $\pi_R^\star$ be the policy defined in \eqref{eq:reduced_cmdp}.  Under Assumptions~\ref{assump:cascade_structure}--\ref{assump:lipschitz_PR}, the transferred policy guarantees that the full-order closed-loop system remains safe up to time $T$ with probability at least:
\begin{align}\label{eq:final_safety_bound}
&\mathbb{P}_{K}\left(\bigcap_{t=0}^{T} \{S_t \in \mathcal{S}_{\mathrm{safe}}\} \mid \pi_R^\star\right) \ge \\& \qquad \qquad  \qquad \qquad 1 - \delta - \frac{L  }{1-\alpha} \bigl( e_0 + \beta    D \bigr). \nonumber 
\end{align}
\end{theorem}
\begin{proof}
    See Section~\ref{thm_proof}
\end{proof}
This theorem demonstrates that the safety degradation under transfer is strictly bounded by the parameters of the inner-loop controller. Specifically, a controller with superior tracking capabilities (smaller $\alpha$ and $\beta$) yields tighter bounds, while a smaller initial tracking error $e_0$ improves the deployment safety guarantee. This intuition is analogous to the one derived in \cite{11107987} for the expected task performance (value function). Furthermore, to address the impact of the reference variation, we note that the safety guarantee depends on how aggressively the policy changes its commanded reference over time. In particular, the result can be used to guide the choice 
of the reduced-order safety level: to guarantee a desired safety probability after transfer, one may impose a stricter constraint in the reduced-order model so as to compensate for the additional degradation term introduced by imperfect tracking.

We prove Theorem \ref{thm:probabilistic_safety_new} in the next section and we evaluate it numerically for a quadrotor navigating an environment in Section \ref{sec_numerical_examples}.  Before doing so, we present  a pertinent remark regarding the bound on the reference variation.

\section{Analysis}\label{sec:analysis}

We begin this section by analyzing the discrepancy of the reduced-order and closed-loop full system transition kernels.
Let $\tau_T = (S_0, A_0, X_0^\star, \dots, S_T, A_T, X_T^\star)$ denote a trajectory of length $T$. We define $\mathbb{P}_{K}^{\pi_R^\star}(\tau_T)$ and $\mathbb{P}_{R}^{\pi_R^\star}(\tau_T)$ as the joint probability distributions over the trajectories of the full-order closed-loop system and the reduced-order model, respectively. 

By applying Bayes' rule and leveraging the Markov property, we can recursively define the trajectory distribution for the reduced-order model as
\begin{align}\label{eq:traj_R}
\mathbb{P}_{R}^{\pi_R^\star}(\tau_t) ={}& \mathbb{P}_{R}^{\pi_R^\star}(\tau_{t-1}) \mathbb{P}_{R}(S_t \mid S_{t-1}, A_{t-1}, X_{t-1}^\star) \nonumber \\
&\times \pi_R^\star(A_t, X_t^\star \mid S_t).
\end{align}
Analogously, the trajectory distribution for the closed-loop system is given by
\begin{align}\label{eq:traj_K}
\mathbb{P}_{K}^{\pi_R^\star}(\tau_t) ={}& \mathbb{P}_{K}^{\pi_R^\star}(\tau_{t-1}) \mathbb{P}_{K}(S_t \mid S_{t-1}, A_{t-1}, X_{t-1}^\star) \nonumber \\
&\times \pi_R^\star(A_t, X_t^\star \mid S_t).
\end{align}

To quantify the discrepancy between these models, we define the trajectory divergence $\Delta P(t)$ as the total variation distance between their respective joint distributions:
\begin{equation}\label{eq:delta_P_traj}
\Delta P(t)
:=
\left\|
\mathbb{P}_{K}^{\pi_R^\star}
-
\mathbb{P}_{R}^{\pi_R^\star}
\right\|_{TV}.
\end{equation}

The following lemma establishes how the one-step total variation $\Delta_{t+1}$ (Definition \ref{def:delta_cdc}) compounds over the trajectory.

\begin{lemma}[Finite-Horizon Trajectory Divergence]\label{lem:traj_divergence}
Under Assumption~\ref{assump:initial_state}, the trajectory divergence over the horizon $T$ is bounded by the sum of the one-step total variation mismatches:
\begin{equation}
\Delta P(T) \le \sum_{t=0}^{T-1} \Delta_{t+1}.
\end{equation}
\end{lemma}
\begin{proof}\label{lemma_proof}
Following a recursive expansion analogous to \cite[Appendix C]{11107987}, we substitute \eqref{eq:traj_R} and \eqref{eq:traj_K} into the definition of $\Delta P(t)$. By separating the sum over the trajectory $\tau_t$ into the history $\tau_{t-1}$, the current state $S_t$, and the current actions $(A_t, X_t^\star)$, we can factor out the shared policy term:
\begin{align}
\Delta P(t) &= \frac{1}{2}\sum_{\tau_{t-1}, S_t} \underbrace{\left[ \sum_{A_t, X_t^\star} \pi_R^\star(A_t, X_t^\star \mid S_t) \right]}_{= 1} \nonumber \\
&\quad \times \left| \mathbb{P}_{K}^{\pi_R^\star}(\tau_{t-1})\mathbb{P}_{K} - \mathbb{P}_{R}^{\pi_R^\star}(\tau_{t-1})\mathbb{P}_{R} \right|.
\end{align}
Because the policy $\pi_R^\star$ is a valid conditional probability distribution, the sum over the action space evaluates to one. Next, by adding and subtracting the cross-term $\mathbb{P}_{R}(S_t \mid S_{t-1}, \dots) \mathbb{P}_{K}^{\pi_R^\star}(\tau_{t-1})$ to the remaining absolute difference and applying the triangle inequality, we obtain:
\begin{align}
\Delta P(t) \le{}& \frac{1}{2}\sum_{\tau_{t-1}, S_t} \mathbb{P}_{K}^{\pi_R^\star}(\tau_{t-1}) \left| \mathbb{P}_{K} - \mathbb{P}_{R} \right| \nonumber \\
& + \frac{1}{2}\sum_{\tau_{t-1}, S_t} \mathbb{P}_{R} \left| \mathbb{P}_{K}^{\pi_R^\star}(\tau_{t-1}) - \mathbb{P}_{R}^{\pi_R^\star}(\tau_{t-1}) \right|.
\end{align}

Recognizing that the first sum is bounded by the maximum one-step total variation mismatch $\Delta_t$ (as defined in \eqref{eq:delta_t_def}), and the second sum simplifies to $\Delta P(t-1)$ because $\mathbb{P}_R$ sums to one over $S_t$, the above equation reduces to 
\begin{equation}
\Delta P(t) \le \Delta_t + \Delta P(t-1).
\end{equation}
Applying this inequality recursively from $t=1$ to $T$, and noting that $\Delta P(0) = 0$ due to the identical initial state distributions (see Assumption~\ref{assump:initial_state}), yields the desired result. \hfill \QED 
\end{proof}

Equipped with the trajectory divergence bound we are in conditions to prove Theorem \ref{thm:probabilistic_safety_new}.

\begin{proof}[Proof of Theorem~\ref{thm:probabilistic_safety_new}]\label{thm_proof}
We begin by establishing the probability of safety on the reduced-order model. Although this implication appears in~\cite{9718160}, we include it here for compactness. By the union bound, the probability of the reduced-order model failing before time $T$ is bounded by its expected cumulative cost:
\begin{align}
\mathbb{P}_{R}\left(\bigcup_{t=0}^{T} \{S_t \notin \mathcal{S}_{\mathrm{safe}}\}\right)
&\leq \sum_{t=0}^{T} \mathbb{P}_{R}(S_t \notin \mathcal{S}_{\mathrm{safe}}) \nonumber\\
&\leq  J_c^{R}(\pi_R^\star) \leq \delta, \label{eq:union_bound_safety}
\end{align}
where the last inequality follows from the definition of $\pi_R^\star$, see \eqref{eq:reduced_cmdp}. Thus, it follows that
\begin{equation}\label{eq:reduced_safety_prob}
\mathbb{P}_R\left(\bigcap_{t=0}^{T} \{S_t \in \mathcal{S}_{\mathrm{safe}}\}\right) \ge 1 - \delta.
\end{equation}

We define to simplify the notation 
$$\mathcal{E} := \bigcap_{t=0}^{T} \{S_t \in \mathcal{S}_{\mathrm{safe}}\}.$$

Note that the probability of $\mathcal{E}$ depends exclusively on the marginal distribution of the outer-state trajectory, $S_{0:T} = (S_0, \dots, S_T)$. Let $\mathbb{P}_{K}^{S_{0:T}}$ and $\mathbb{P}_{R}^{S_{0:T}}$ denote these marginal distributions under the closed-loop and reduced-order models, respectively.

By the definition of total variation distance (see e.g., \cite[Chap.~3]{DBLP:books/cu/10/D2010}), the absolute difference in the probability of $\mathcal{E}$ under two probability measures is upper bounded by the total variation distance between the marginal state distributions
\begin{equation}\label{eq:event_tv_bound}
\left| \mathbb{P}_{K}(\mathcal{E}) - \mathbb{P}_{R}(\mathcal{E}) \right|
\le
\left\| \mathbb{P}_{K}^{S_{0:T}} -\mathbb{P}_{R}^{S_{0:T}}\right\|_{TV}.
\end{equation}

Furthermore, since the outer-state trajectory is a marginal of the full joint trajectory $\tau_T$, the total variation distance between the corresponding marginal distributions is upper bounded by the total variation distance between the corresponding joint distributions~\cite{Wai19}, i.e., 
\begin{equation}\label{eq:tv_marginal_joint}
\left\| \mathbb{P}_{K}^{S_{0:T}} -\mathbb{P}_{R}^{S_{0:T}}\right\|_{TV}
\le
\left\| \mathbb{P}_{K}^{\pi_R^\star}(\tau_T)- \mathbb{P}_{R}^{\pi_R^\star}(\tau_T)\right\|_{TV}.
\end{equation}

From the definition of the trajectory divergence $\Delta P(T)$ in \eqref{eq:delta_P_traj}, the difference between the probabilities of the event $\mathcal{E}$ under the two models is bounded by $\Delta P(T)$
\begin{equation}
\left| \mathbb{P}_{K}(\mathcal{E}) - \mathbb{P}_{R}(\mathcal{E}) \right|
\le \Delta P(T).
\end{equation}

To complete the proof we are left to show that $\Delta P(T) \leq L(e_0+\beta D)/(1-\alpha)$.  
It follows from Lemma \ref{lem:traj_divergence} and Proposition \ref{prop:delta_bound_cdc} that
\begin{equation}\label{eq:sum_delta_t}
\Delta P(T) \le \sum_{t=0}^{T-1} L  \left( \alpha^t e_0 + \beta \sum_{\ell=1}^{t}\alpha^{t-\ell} d_\ell \right).
\end{equation}

We evaluate the two sums separately. The leftmost sum is upper bounded by the infinite geometric series
\begin{equation}
\sum_{t=0}^{T-1} \alpha^t e_0 \le e_0 \sum_{t=0}^{\infty} \alpha^t = e_0 \frac{1}{1-\alpha}.
\end{equation}

For the reference variation term, exchanging the order of summations yields
\begin{equation}\label{eqn_aux_theo1}
\sum_{t=1}^{T-1} \sum_{\ell=1}^{t} \alpha^{t-\ell} d_\ell = \sum_{\ell=1}^{T-1} d_\ell \sum_{t=\ell}^{T-1} \alpha^{t-\ell}.
\end{equation}

The inner sum is a finite geometric series. Therefore, it is upper bounded by $1/(1-\alpha)$. Applying this to \eqref{eqn_aux_theo1} yields
\begin{equation}
\sum_{\ell=1}^{T-1} d_\ell \sum_{k=0}^{T-1-\ell} \alpha^k \le \frac{1}{1-\alpha} \sum_{\ell=1}^{T-1} d_\ell.
\end{equation}

Assumption~\ref{assump:reference_variation_cdc} and the nonnegativity of $\{d_t\}_{t=1}^T$ yields
\begin{equation}
\sum_{\ell=1}^{T-1} d_\ell \le \sum_{\ell=1}^{T} d_\ell \le D.
\end{equation}

Combining these sums back into \eqref{eq:sum_delta_t} yields the stated bound for $\Delta P(t)$ thus, completing the proof of the result. \hfill \QED 
\end{proof}
\section{Experimental Results}\label{sec_numerical_examples}
In this section, we evaluate our theoretical findings by training a constrained policy on a reduced-order model and deploying it on the full-order system using an inner-loop controller to track the inner state reference. In contrast to~\cite{11107987}, which focuses on reward degradation after transfer, our goal here is to assess whether the reduced-order safety guarantees persist on the full-order dynamics.

The state of the full-order planar quadrotor is given by $[p_x,\dot{p}_x,p_z,\dot{p}_z,\theta,\dot{\theta}]^\top$, where $p_x$ and $p_z$ denote the horizontal and vertical positions of the center of mass, respectively, and $\theta$ is the pitch angle. The control inputs are the collective thrust $F$ and the body pitch moment $M$. For notational simplicity, we omit the explicit time dependence $(t)$ for all continuous-time variables. The continuous-time translational dynamics are:
\begin{align}
\ddot{p}_x &= \frac{F}{m}\sin(\theta), &
\ddot{p}_z &= \frac{F}{m}\cos(\theta)-g,
\label{eq:full_quad_trans}
\end{align}
while the rotational dynamics are governed by $J\ddot{\theta} = M$.

Throughout the experiments, we use a mass $m=1\,\mathrm{kg}$, gravity $g=9.81\,\mathrm{m/s^2}$, resulting in a nominal hover thrust of $F_{\mathrm{hover}} = mg = 9.81\,\mathrm{N}$. The pitch moment of inertia is $J=0.02\,\mathrm{kg\,m^2}$. We discretize the dynamics using semi-implicit Euler integration with a sampling time of $\Delta t = 0.05\,\mathrm{s}$.

We formulate a lower-dimensional training model by replacing the physical attitude state $\theta_t$ with a commanded reference $\theta_t^\star$, treated as an input. The resulting reduced-order model retains only the translational state $S_t = [p_{x,t},\dot{p}_{x,t},p_{z,t},\dot{p}_{z,t}]^\top$. The RL agent selects the action $A_t = [\Delta F_t,\theta_t^\star]^\top$, where $\Delta F_t$ is the thrust increment around hover ($F_t = F_{\mathrm{hover}}+\Delta F_t$). Substituting the input $\theta_t^\star$ into the translational accelerations from \eqref{eq:full_quad_trans} and applying semi-implicit Euler discretization yields the discrete reduced-order model
\begin{align}
\dot{p}_{x,t+1} &= \dot{p}_{x,t} + \ddot{p}_{x,t} \Delta t, &
p_{x,t+1} &= p_{x,t} + \dot{p}_{x,t+1}\Delta t, \nonumber\\
\dot{p}_{z,t+1} &= \dot{p}_{z,t} + \ddot{p}_{z,t} \Delta t, &
p_{z,t+1} &= p_{z,t} + \dot{p}_{z,t+1}\Delta t,
\label{eq:reduced_discrete_rewrite}
\end{align}
where $\ddot{p}_{x,t} = F_t\sin(\theta_t^\star)/m$ and $\ddot{p}_{z,t} = {F_t}\cos(\theta_t^\star)/m-g$.
\subsection{Reward, Safety Cost, and PPO-Lagrangian Training}

The task is to drive the vehicle toward the goal position $(p_{x,g},p_{z,g})=(9,9)$, while remaining in the safe set $\mathcal{S}_{\mathrm{safe}}=\{(p_x,p_z): p_x\le 9\}$. The reward is defined as:
\begin{align*}
r_t
&=
-\frac{(p_{x,t}-p_{x,g})^2+(p_{z,t}-p_{z,g})^2}{100} \\
&\quad
+10\,\mathbf{1}\!\left\{\sqrt{(p_{x,t}-p_{x,g})^2+(p_{z,t}-p_{z,g})^2}<0.1\right\},\nonumber
\label{eq:task_reward_rewrite}
\end{align*}
which penalizes the distance to the goal and provides a bonus upon reaching a small neighborhood of the target. The corresponding stage safety cost is the indicator $c_t = \mathbf{1}\{p_{x,t}>9\}$.

The policy is trained on this reduced-order CMDP using PPO-Lagrangian \cite{ray2019benchmarking}. At training iteration $k$, the algorithm maximizes the penalized reward
\[
\tilde r_t = r_t - \lambda_k c_t,
\]
where $\lambda_k\ge 0$ is the dual variable. Let
\[
\bar C_{\mathrm{ep},\gamma}
:=
\frac{1}{T}\sum_{t=0}^{T-1}\gamma^t c_t
\]
denote the empirical normalized discounted safety cost over an episode. The multiplier is then updated by dual ascent according to
\[
\lambda_{k+1}
=
\max\!\left(0,\lambda_k+\eta_\lambda\bigl(\bar C_{\mathrm{ep},\gamma}-\delta\bigr)\right).
\]
Although the theoretical development uses an undiscounted finite-horizon safety cost, the PPO-Lagrangian implementation employs this discounted surrogate to improve optimization stability and temporal credit assignment during training. Since the horizon is finite and the discount factor is chosen close to one, this surrogate remains close to the undiscounted cost used in the analysis. 

In the implementation, the failure tolerance is set to $\delta=0.025$, discount factor $\gamma=0.994$, initial multiplier $\lambda_0=1$, step size $\eta_\lambda=0.02$, learning rate $4\times 10^{-4}$, and PPO clip range $0.05$.

\subsection{Deployment on the Full-Order System with Inner-Loop Attitude Control}

After training, the policy $\pi_R^\star$ is deployed on the full-order quadrotor. The policy observes the translational state $[p_{x,t},\dot{p}_{x,t},p_{z,t},\dot{p}_{z,t}]^\top$ and outputs the action pair $(\Delta F_t,\theta_t^\star)$. Although the transfer analysis is stated for a generic inner-state reference $X_t^\star$, in the quadrotor experiments we further exploit the problem structure and train on an even simpler interface, where the policy outputs only the desired pitch angle $\theta_t^\star$. This reduces the dimension of the learning problem even further, while the corresponding reference-rate signal is reconstructed locally for the inner-loop controller.

To track the commanded reference, the controller in \eqref{eq:conrol_law} is implemented as a proportional-derivative law that outputs the pitch moment $M_t$:
\begin{equation}\label{eq:inner_loop_M}
M_t
=
-K_p(\theta_t-\theta_t^\star)
-K_d(\dot{\theta}_t-\dot{\theta}_{t,\mathrm{est}}^\star),
\end{equation}
where $\dot{\theta}_{t,\mathrm{est}}^\star$ is an estimate of the reference rate. This signal is calculated using a backward difference followed by a low-pass filter:
\begin{align}
\dot{\theta}_{t,\mathrm{raw}}^\star
&=
\frac{\theta_t^\star-\theta_{t-1}^\star}{\Delta t},
\quad  
\dot{\theta}_{t,\mathrm{est}}^\star
=
\alpha_f \dot{\theta}_{t-1,\mathrm{est}}^\star
+
(1-\alpha_f)\dot{\theta}_{t,\mathrm{raw}}^\star,
\label{eq:theta_filter_rewrite}
\end{align}
with $\alpha_f = \exp(-\Delta t/T_f)$ and filter time constant $T_f=0.1\,\mathrm{s}$. The role of this filter is to smooth the high-frequency variation introduced by numerical differentiation of $\theta_t^\star$ while keeping the estimated rate responsive enough for attitude tracking. Since $\Delta t=0.05\,\mathrm{s}$, choosing $T_f=0.1\,\mathrm{s}$ provides a simple compromise between noise attenuation and additional phase lag in the inner loop. 

The safety guarantees are evaluated across different inner-loop gains. The inner-loop gains are parameterized by the natural frequency $\omega_n$ and the damping ratio $\zeta$:
$$
K_p = J\omega_n^2,
\qquad
K_d = 2J\zeta\omega_n.
$$
The natural frequency ranges from $\omega_n = 2$ to $12$\,rad/s, and the damping ratio ranges from $\zeta = 0.2$ to $1.0$. Each gain pair is evaluated over $N=100$ initial conditions. The initial positions are sampled from $p_{x,0} \in \{1,\dots,8\}$ and $p_{z,0} \in \{1,\dots,9\}$. The initial velocities are sampled uniformly as $\dot{p}_{x,0}, \dot{p}_{z,0} \sim \mathcal{U}(-1,1)$. The initial rotational state is set to $[\theta_0,\dot{\theta}_0]^\top =[0,0]^\top$. 

Figure~\ref{fig:failure_heatmap} shows the probability that an episode becomes unsafe over the finite horizon. This quantity is estimated by the empirical failure rate:
\begin{align*}
\hat p_{\mathrm{fail}}
=
\frac{1}{N}
\sum_{i=1}^{N}
\mathbf{1}\!\left\{
\sum_{t=0}^{T-1} c_t^{(i)} \ge 1
\right\}.
\label{eq:empirical_fail_prob}
\end{align*}

Figure~\ref{fig:error_heatmap} shows the mean attitude tracking error:
\[
\bar{e}_\theta = \frac{1}{T}\sum_{t=0}^{T-1} |\theta_t - \theta_t^\star|
\]
for the same gain values.

\begin{figure}[htbp]
    \centering
    \includegraphics[width=\linewidth]{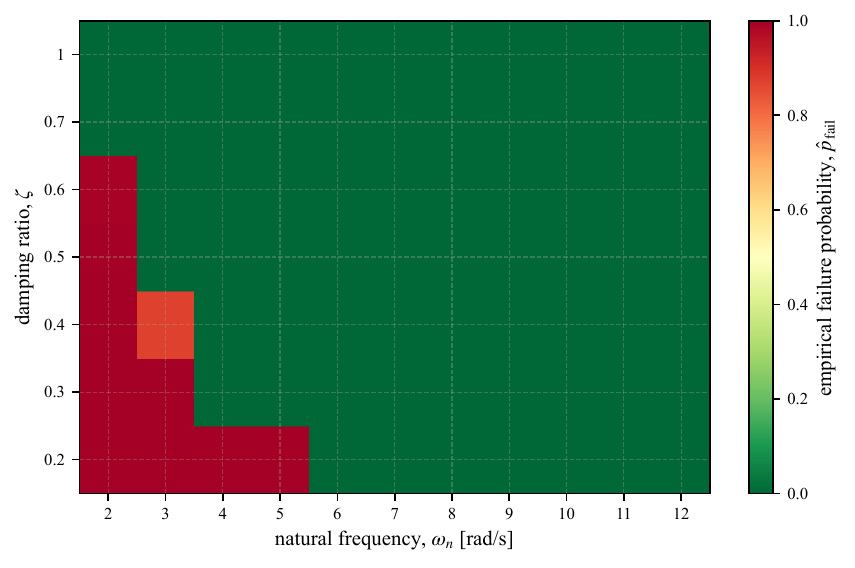}
    \caption{Empirical failure probability $\hat p_{\mathrm{fail}}$ of the transferred policy across inner-loop natural frequencies $\omega_n \in [2, 12]$\,rad/s and damping ratios $\zeta \in [0.2, 1.0]$. Safety violations occur in the region where $\omega_n \le 5$\,rad/s and $\zeta \le 0.6$. Outside this parameter region, the failure probability evaluates to 0.0.}
    \label{fig:failure_heatmap}
\end{figure}

\begin{figure}[htbp]
    \centering
    \includegraphics[width=\linewidth]{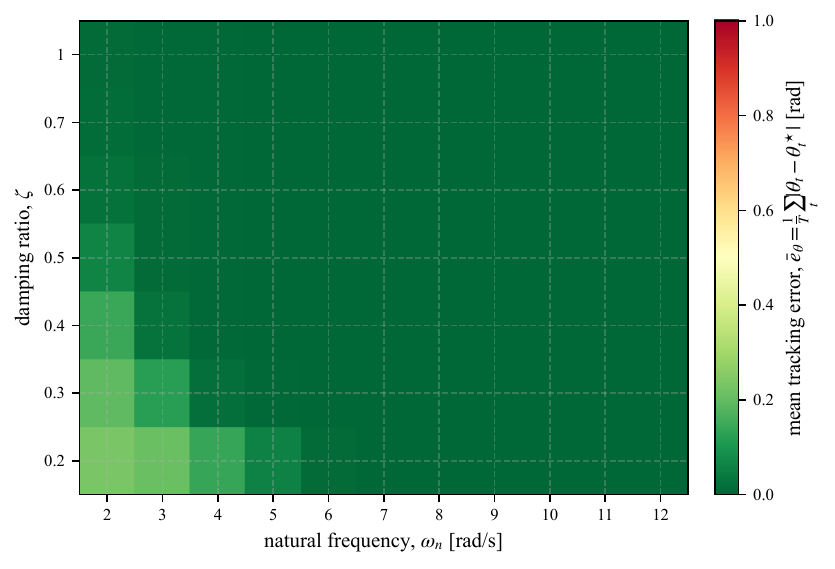}
    \caption{Mean attitude tracking error $\bar{e}_\theta = \frac{1}{T}\sum_{t=0}^{T-1} |\theta_t - \theta_t^\star|$ across inner-loop natural frequencies $\omega_n \in [2, 12]$\,rad/s and damping ratios $\zeta \in [0.2, 1.0]$. The elevated tracking error corresponds directly to the safety degradation observed in Figure~\ref{fig:failure_heatmap}. }
    \label{fig:error_heatmap}
\end{figure}

As shown in Figure~\ref{fig:error_heatmap}, the mean attitude tracking error is highest for small $\omega_n$ and small $\zeta$, and it decreases as either parameter increases. Figure~\ref{fig:failure_heatmap} demonstrates that safety violations appear in this exact same parameter region. As $\omega_n$ and $\zeta$ increase, both the tracking error and the empirical failure probability decrease. For larger inner-loop gains, the failure probability drops to zero across all tested initial conditions. These results validate Theorem~\ref{thm:probabilistic_safety_new}: controllers with poor tracking of the reference $\theta_t^\star$ are associated with a higher probability of failure, whereas improved tracking leads to safer deployment on the full-order system. 

At the same time, these experiments highlight that evaluating the contribution of the reference-variation term is subtle. It is worth noting, however, that obtaining a consistent empirical evaluation of the reference signal variation remains challenging. In practice, the dynamics of the closed-loop system inherently intertwine the effects of the inner-loop controller with the reference signal generated by the policy, making it difficult to perfectly isolate the signal's independent variation. Exploring algorithmic constraints or architectural methods to strictly guarantee a bounded variation on the reference signal during deployment stands as an important direction for future work. 

\section{Conclusion}\label{sec_conclusion}
In this paper, we developed a framework to analyze the zero-shot safety guarantees from reduced-order models to full-order cascade dynamical systems. By formulating the RL training on the reduced-order model as a constrained Markov decision process, we derived a finite-horizon lower bound on the probability of safe deployment for the transferred policy. Our analysis explicitly shows how transfer-induced safety degradation depends on the tracking performance of the inner-loop controller, the initial tracking error, and the cumulative variation of the reference signal.

We validated these results on a quadrotor navigation task, showing that inner-loop controllers with better tracking performance are more effective at preserving the safety guarantees obtained during reduced-order training. The experiments also highlight that the effect of reference variation is subtle, since the commanded reference sequence is generated in closed loop and is therefore intertwined with the system trajectory during deployment. An important direction for future work is to develop policy optimization and control design methods that explicitly constrain or certify reference variation, thereby yielding tighter transfer-safety guarantees for real-world deployment.
\section*{Acknowledgment}
This work was sponsored by the Office of Naval Research (ONR), under contract number N00014-23-1-2377.

\bibliographystyle{ieeetr}
\bibliography{bib}

@inproceedings{molnar2023safetycritical,
  author    = {Tam{\'a}s G. Molnar and Aaron D. Ames},
  title     = {Safety-Critical Control with Bounded Inputs via Reduced Order Models},
  booktitle = {2023 American Control Conference (ACC)},
  pages     = {1414--1421},
  year      = {2023},
  publisher = {IEEE},
  doi       = {10.23919/ACC55779.2023.10155871}
}

@misc{ray2019benchmarking,
  title         = {Benchmarking Safe Exploration in Deep Reinforcement Learning},
  author        = {Alex Ray and Joshua Achiam and Dario Amodei},
  year          = {2019},
  eprint        = {1910.01708},
  archivePrefix = {arXiv},
  primaryClass  = {cs.LG}
}

@article{mnih2015dqn,
  author  = {Volodymyr Mnih and Koray Kavukcuoglu and David Silver and Andrei A. Rusu and Joel Veness and Marc G. Bellemare and Alex Graves and Martin Riedmiller and Andreas K. Fidjeland and Georg Ostrovski and Stig Petersen and Charles Beattie and Amir Sadik and Ioannis Antonoglou and Helen King and Dharshan Kumaran and Daan Wierstra and Shane Legg and Demis Hassabis},
  title   = {Human-level control through deep reinforcement learning},
  journal = {Nature},
  volume  = {518},
  number  = {7540},
  pages   = {529--533},
  year    = {2015},
  doi     = {10.1038/nature14236}
}

@article{silver2016alphago,
  author  = {David Silver and Aja Huang and Chris J. Maddison and Arthur Guez and Laurent Sifre and George van den Driessche and Julian Schrittwieser and Ioannis Antonoglou and Veda Panneershelvam and Marc Lanctot and Sander Dieleman and Dominik Grewe and John Nham and Nal Kalchbrenner and Ilya Sutskever and Timothy Lillicrap and Madeleine Leach and Koray Kavukcuoglu and Thore Graepel and Demis Hassabis},
  title   = {Mastering the game of Go with deep neural networks and tree search},
  journal = {Nature},
  volume  = {529},
  number  = {7587},
  pages   = {484--489},
  year    = {2016},
  doi     = {10.1038/nature16961}
}

@article{silver2017alphagozero,
  author  = {David Silver and Julian Schrittwieser and Karen Simonyan and Ioannis Antonoglou and Aja Huang and Arthur Guez and Thomas Hubert and Lucas Baker and Matthew Lai and Adrian Bolton and Yutian Chen and Timothy Lillicrap and Fan Hui and Laurent Sifre and George van den Driessche and Thore Graepel and Demis Hassabis},
  title   = {Mastering the game of Go without human knowledge},
  journal = {Nature},
  volume  = {550},
  number  = {7676},
  pages   = {354--359},
  year    = {2017},
  doi     = {10.1038/nature24270}
}

@misc{schulman2017ppo,
  author        = {John Schulman and Filip Wolski and Prafulla Dhariwal and Alec Radford and Oleg Klimov},
  title         = {Proximal Policy Optimization Algorithms},
  year          = {2017},
  eprint        = {1707.06347},
  archivePrefix = {arXiv},
  primaryClass  = {cs.LG}
}

@inproceedings{
zubia2025robust,
title={Robust Transfer of Safety-Constrained Reinforcement Learning Agents},
author={Markel Zubia and Thiago D. Sim{\~a}o and Nils Jansen},
booktitle={The Thirteenth International Conference on Learning Representations},
year={2025},
url={https://openreview.net/forum?id=rvXdGL4pCJ}
}

@inproceedings{
as2025spidr,
title={{SP}i{DR}: A Simple Approach for Zero-Shot Safety in Sim-to-Real Transfer},
author={Yarden As and Chengrui Qu and Benjamin Unger and Dongho Kang and Max van der Hart and Laixi Shi and Stelian Coros and Adam Wierman and Andreas Krause},
booktitle={The Thirty-ninth Annual Conference on Neural Information Processing Systems},
year={2025},
url={https://openreview.net/forum?id=Pe1ypX9gBO}
}

@misc{tessler2018reward,
  title={Reward Constrained Policy Optimization},
  author={Chen Tessler and Daniel J. Mankowitz and Shie Mannor},
  year={2018},
  eprint={1805.11074},
  archivePrefix={arXiv},
  primaryClass={cs.LG}
}

@inproceedings{chen2024rom,
  title={Reinforcement Learning for Reduced-order Models of Legged Robots},
  author={Chen, Yu-Ming and Bui, Hien and Posa, Michael},
  booktitle={2024 IEEE International Conference on Robotics and Automation (ICRA)},
  year={2024}
}

@inproceedings{zhao2020simtoreal,
  title={Sim-to-real transfer in deep reinforcement learning for robotics: a survey},
  author={Zhao, Wenshuai and Queralta, Jorge Pe{\~n}a and Westerlund, Tomi},
  booktitle={2020 IEEE Symposium Series on Computational Intelligence (SSCI)},
  pages={737--744},
  year={2020},
  organization={IEEE}
}

@INPROCEEDINGS{11107987,
  author={Rabiei, Shima and Mishra, Sandipan and Paternain, Santiago},
  booktitle={2025 American Control Conference (ACC)}, 
  title={Transfer Learning for a Class of Cascade Dynamical Systems}, 
  year={2025},
  volume={},
  number={},
  pages={231-238},
  doi={10.23919/ACC63710.2025.11107987}}

@inproceedings{achiam2017cpo,
  title={Constrained policy optimization},
  author={Achiam, Joshua and Held, David and Tamar, Aviv and Abbeel, Pieter},
  booktitle={International conference on machine learning},
  pages={22--31},
  year={2017},
  organization={PMLR}
}

@inproceedings{cheng2019barrier,
  title={End-to-end safe reinforcement learning through barrier functions for safety-critical continuous control tasks},
  author={Cheng, Richard and Orosz, G{\'a}bor and Murray, Richard M and Burdick, Joel W},
  booktitle={Proceedings of the AAAI Conference on Artificial Intelligence},
  volume={33},
  pages={3387--3395},
  year={2019}
}

@article{wabersich2021predictive,
  title={A predictive safety filter for learning-based control of constrained nonlinear dynamical systems},
  author={Wabersich, Kim P and Zeilinger, Melanie N},
  journal={Automatica},
  volume={129},
  pages={109597},
  year={2021},
  publisher={Elsevier}
}

@article{hsu2023simlabreal,
  title={Sim-to-Lab-to-Real: Safe reinforcement learning with shielding and generalization guarantees},
  author={Hsu, Kai-Chieh and Ren, Allen Z. and Nguyen, Duy P. and Majumdar, Anirudha and Fisac, Jaime F.},
  journal={Artificial Intelligence},
  volume={314},
  pages={103811},
  year={2023}
}

@article{zhang2024safetytransfer,
  title={Safety reinforcement learning control via transfer learning},
  author={Zhang, Quanqi and Wu, Chengwei and Tian, Haoyu and Gao, Yabin and Yao, Weiran and Wu, Ligang},
  journal={Automatica},
  volume={166},
  pages={111714},
  year={2024},
  doi={10.1016/j.automatica.2024.111714}
}

@book{DBLP:books/cu/10/D2010,
  author       = {Rick Durrett},
  title        = {Probability: Theory and Examples, 4th Edition},
  publisher    = {Cambridge University Press},
  year         = {2010},
  url          = {https://doi.org/10.1017/CBO9780511779398},
  doi          = {10.1017/CBO9780511779398},
  isbn         = {9780511779398},
  bibsource    = {dblp computer science bibliography, https://dblp.org}
}

@book{Wai19,
 address = {Cambridge, UK},
 author = {M. J. Wainwright},
 publisher = {Cambridge University Press},
 title = {High-dimensional statistics: A non-asymptotic viewpoint},
 year = {2019}
}

@article{bhatnagar2005actor,
  title={An actor-critic algorithm for constrained Markov decision processes},
  author={Bhatnagar, Shalabh and Lakshmikanthan, K},
  journal={Systems \& control letters},
  volume={54},
  number={10},
  pages={1011--1022},
  year={2005},
  publisher={Elsevier}
}

@ARTICLE{9718160,
  author={Paternain, Santiago and Calvo-Fullana, Miguel and Chamon, Luiz F. O. and Ribeiro, Alejandro},
  journal={IEEE Transactions on Automatic Control}, 
  title={Safe Policies for Reinforcement Learning via Primal-Dual Methods}, 
  year={2023},
  volume={68},
  number={3},
  pages={1321-1336},
  doi={10.1109/TAC.2022.3152724}}

@misc{paternain2019constrainedreinforcementlearningzero,
      title={Constrained Reinforcement Learning Has Zero Duality Gap}, 
      author={Santiago Paternain and Luiz F. O. Chamon and Miguel Calvo-Fullana and Alejandro Ribeiro},
      year={2019},
      eprint={1910.13393},
      archivePrefix={arXiv},
      primaryClass={cs.LG},
      url={https://arxiv.org/abs/1910.13393}, 
}

@book{Khalil:1173048,
      author        = "Khalil, Hassan K",
      title         = "{Nonlinear systems; 3rd ed.}",
      publisher     = "Prentice-Hall",
      address       = "Upper Saddle River, NJ",
      year          = "2002",
}

@article{devroye2018total,
  title={The total variation distance between high-dimensional Gaussians with the same mean},
  author={Devroye, Luc and Mehrabian, Abbas and Reddad, Tommy},
  journal={arXiv preprint arXiv:1810.08693},
  year={2018}
}

@article{chen2024probabilistic,
  title={Probabilistic constraint for safety-critical reinforcement learning},
  author={Chen, Weiqin and Subramanian, Dharmashankar and Paternain, Santiago},
  journal={IEEE Transactions on Automatic Control},
  year={2024},
  publisher={IEEE}
}

@book{altman1999constrained,
  title={Constrained Markov Decision Processes},
  author={Altman, Eitan},
  year={1999},
  publisher={Chapman \& Hall/CRC},
  address={Boca Raton, FL, USA},
  isbn={9780849303371}
}

\end{document}